\documentclass[review,3p,times]{elsarticle}

\usepackage{amssymb}
\usepackage{amsmath}
\usepackage{bm}
\usepackage[dutch,english]{babel}
\usepackage{soul}
\usepackage{booktabs}
\usepackage{multirow}
\usepackage{mathtools}
\usepackage{tabularx}
\usepackage{enumerate}
\usepackage{nicefrac}
\usepackage{caption}
\usepackage{subcaption}
\usepackage{natbib}

\usepackage{multirow}
\usepackage{dsfont}
\usepackage{makecell}
\usepackage{enumitem}
\setlist{nolistsep}
\usepackage[utf8]{inputenc}
\usepackage[english]{babel}
\usepackage{amsmath}
\usepackage{amssymb}
\usepackage{amsthm}

\usepackage{mathrsfs}

\usepackage{color}

\newtheorem{theorem}{Theorem}
\newtheorem{proposition}{Proposition}
\newtheorem{lemma}{Lemma}

\theoremstyle{remark}
\newtheorem{remark}{Remark}

\theoremstyle{definition}

\def\N{\mathbb{N}}

\def\Z{\mathbb{Z}}
\def \E {\mathbb{E}}

\def \P {\mathbb{P}}

\def \-> {\rightarrow}

\journal{Operations Research Letters}

\begin{document}

\begin{frontmatter}

\title{Optimal data pooling for shared learning in maintenance operations}

\author{Collin Drent\corref{cor1}}

\cortext[cor1]{Corresponding author}
\ead{c.drent@tue.nl}

\author{Melvin Drent}
\ead{m.drent@tue.nl}

\author{Geert-Jan van Houtum}
\ead{g.j.v.houtum@tue.nl}

\address{Eindhoven University of Technology, School of Industrial Engineering, Eindhoven, The Netherlands}
\begin{abstract}
We study optimal data pooling for shared learning in two common maintenance operations: condition-based maintenance and spare parts management. We consider a set of systems subject to Poisson input -- the degradation or demand process -- that are coupled through an a-priori unknown rate. Decision problems involving these systems are high-dimensional Markov decision processes (MDPs) and hence notoriously difficult to solve. We present a decomposition result that reduces such an MDP to two-dimensional MDPs, enabling structural analyses and computations. Leveraging this decomposition, we (i) demonstrate that pooling data can lead to significant cost reductions compared to not pooling, and (ii) show that the optimal policy for the condition-based maintenance problem is a control limit policy, while for the spare parts management problem, it is an order-up-to level policy, both dependent on the pooled data.
\end{abstract}
\begin{keyword}
condition-based maintenance \sep data pooling \sep Bayesian learning \sep spare parts \sep optimal policy

\end{keyword}

\end{frontmatter}


\section{Introduction}
Unplanned downtime of advanced technical systems such as aircraft, lithography systems, or rolling stock, is extremely costly for companies that rely on these systems in their primary processes. As such, these companies typically have agreements with maintenance service providers -- external or internal -- to ensure sufficiently high availability of their systems.
Recent advancements in information technology along with continuous reductions in costs of sensors have led to ample opportunities for service providers to improve their maintenance operations \citep{olsen2020industry}. Indeed, modern systems are now increasingly equipped with sensors that relay degradation data of critical components in real-time to maintenance decision-makers. This data is useful for inference of system degradation behavior; however, the amount of data that each such system generates  to predict failures of a particular component is scarce, especially for newly introduced systems. 

Maintenance service providers typically maintain several systems of the same type (e.g. similar systems for the same customer at different locations, or similar systems for different customers).
At the beginning of the life-cycle of a newly introduced system, the maintenance service provider thus faces a setting where (i) multiple systems of the same type generate a steady stream of degradation data, but at the same time, (ii) each such system alone has not yet generated sufficient amounts of data. 
A prime example of this can be found in the semiconductor industry, where the original equipment manufacturer itself is often also responsible for maintaining its lithography systems after they are sold. Upon the introduction of a new generation of lithography system in the field, many critical components in this system are also used for the first time, and hence no historical degradation data is available \citep[][]{dursun2022data}.  

For the settings described above, it is evident that pooling degradation data from multiple systems can lead to cost reductions in maintenance operations. However, it remains unclear how we can precisely quantify these cost reductions, especially when we are interested in optimal decisions and the state space of the corresponding Markov decision process (MDP) thus becomes large. In this paper we address this question. More specifically, we consider a maintenance service provider that is responsible for maintaining multiple systems of the same type at different locations or customers. These systems are equipped with sensory technology that relay degradation data in real-time to the service provider. As these systems are used for the first time, there is only limited information available per system at the start of their life-cycle.

We consider a single component that is present in the configuration of all systems. These components deteriorate according to a Poisson process with the same but unknown rate. As the components are critical, the systems fail whenever the component's degradation reaches a certain failure threshold. Such failures can be prevented by performing preventive maintenance, which is cheaper than replacement upon failure, which generally leads to costly unplanned downtime. The maintenance service provider must periodically decide -- based on the state of all systems and accumulated data -- for each system to perform preventive maintenance or not, thereby trading off costly premature interventions with costly tardy replacements. Systems are homogeneous with respect to the unknown deterioration rate, but are otherwise heterogeneous (i.e. costs and failure thresholds). We endow the unknown rate with a prior distribution and propose a Bayesian procedure that is able to pool all data and learn this rate jointly on-the-fly as data becomes available. We model this decision problem as a Bayesian MDP for which the optimal policy -- in theory -- can be computed through standard methods. However, because both the action and state space grow exponentially in the number of systems, this MDP will quickly suffer from the curse of dimensionality, making it impossible to assess the value of optimal data pooling. As a remedy, we establish a novel decomposition result that reduces this high-dimensional MDP to multiple two-dimensional MDPs that permit structural analyses and computations. 

When components have constant failure rates, maintenance service providers typically replace these components with new spares only correctively upon failure, i.e. they apply repair-by-replacement. The underlying spare parts inventory system responsible for supplying these spares then largely determines the availability of the technical systems. As an extension, we will show that our decomposition result also applies to such a spare parts inventory system consisting of multiple local warehouses that keep spares for the same critical component whose failure rate is unknown.

Sequential Bayesian learning based on sensory data stemming from systems has been used extensively in the maintenance literature to study optimal maintenance decision-making when relevant parameters are a-priori unknown \cite[e.g.][]{elwany2011structured,drent2020censored,drent2022real}, but only exclusively for single-component systems in isolation (we refer to \cite{de2020review} for a comprehensive overview of the area). This makes sense when the unknown parameter is unique to the specific system in use. However, as we argued above, in practice parameters may be the same for multiple systems of the same type. When a maintenance service provider maintains several systems of the same type, as we consider in this paper, it is natural to pool data stemming from all these systems to jointly learn the common parameter. 

The benefit of pooling has been extensively studied in many application domains, yet almost exclusively related to pooling of physical resources. Notable examples include inventory pooling in inventory networks \cite[see, e.g.,][]{eppen1979note} and pooling of server capacity in queuing networks \citep[see, e.g.,][]{mandelbaum1998pooling}. Recently, researchers have started to investigate the benefits of pooling data, mainly driven by the benefits of pooling physical resources \citep[see, e.g.,][]{bastani2021meta,gupta2021data}. 
Within the maintenance literature, only two works exist on data pooling for learning parameters \citep{deprez2022data,dursun2022data}. \cite{deprez2022data} investigates the benefits of combining data from a set of heterogeneous machines in the context of time-based preventive maintenance. The authors propose a method in which limited data stemming from multiple systems can be aggregated such that it can be utilized for selecting a periodic interval at which preventive maintenance is performed for each individual system. \cite{dursun2022data} exploits the use of data pooling to determine whether a set of systems is stemming from a so-called weak or strong population, where the former has lifetimes that are stochastically smaller than the latter. Unlike \cite{deprez2022data}, who proposes a static estimation procedure based on historical pooled data, \cite{dursun2022data}  builds a partially observable MDP that sequentially learns as more data becomes available. They numerically show -- only for small instances due to the curse of dimensionality -- that data pooling can lead to savings of up to 14\% compared to not pooling data. We also learn from pooled data in a dynamic, sequential way, but circumvent the resulting curse of dimensionality by leveraging our new decomposition result. Both \cite{deprez2022data} and \cite{dursun2022data} pool data to learn a time-to-failure model in a time-based maintenance setting, while we focus on learning a degradation model in a condition-based maintenance setting. 

The main contributions of this paper are as follows: 
\begin{enumerate}
\item We formulate the problem of optimally maintaining $N$ systems with a common, unknown deterioration rate over a finite lifespan $T$ as a finite horizon Bayesian MDP in which data is pooled for shared learning. This formulation suffers from the well-known curse of dimensionality: The cardinality of both the action and state space grow exponentially in $N$. As a remedy, we provide a new decomposition result that establishes the equivalence between the original MDP and $N$ two-state MDPs with a binary action space, each focused on an individual system.  
\item Using the decomposition, we are able to show that the structure of the optimal policy of each individual system has a control limit structure, where the control limit depends on the pooled data obtained from all systems. Perhaps counterintuitively, we show that this optimal control limit is not monotone in general. Although the control limit typically decreases first, it always increases and converges to the failure level when the pooled data grows very large, implying that preventive maintenance is never optimal in that asymptotic regime. 
\item We investigate numerically the savings that can be attained by pooling data to learn the a-priori unknown deterioration rate, while optimally maintaining the systems. We find that the savings can be significant, even for small values of $N$, and that the exact magnitude of these savings largely depends on the magnitude of the uncertainty in the parameter (measured by the variation of the initial prior distribution). When there is high uncertainty, huge savings of close to $57\%$ can be realized on average, while these savings become almost negligible when the uncertainty is low.
\item We finally demonstrate the general applicability of our decomposition result by applying it to a spare parts inventory system consisting of multiple local warehouses where a common, but unknown failure rate needs to be learned. For this setting, we establish the optimality of monotone order-up-to policies, where the optimal order-up-to levels are non-decreasing in the data obtained from all local warehouses.
\end{enumerate} 

The remainder is organized as follows. We provide a model description in Section \ref{cbmPoolFormulation}. In Section \ref{poolingMDP}, we formulate the problem as a Bayesian MDP and we show that it can be decomposed into $N$ alternative MDPs. We present some structural properties of both the expected cost and the optimal policy of the alternative MDP in Section \ref{structAlternative}. In Section \ref{PoolingNum}, we report on an extensive numerical study that highlights the benefit of pooling data. In Section \ref{Pooling:spare}, we apply our decomposition result to a set of spare parts inventory systems. Finally, Section \ref{Pooling:conclusion} provides concluding remarks.

\section{Model description}\label{cbmPoolFormulation}
We consider a set of $N\geq 1$ systems subject to damage accumulation due to random shocks that arrive over time. Random shock degradation is a common assumption in the maintenance literature \cite[see, e.g.,][]{kurt2010monotone,li2022after} that has been validated by practice-based research \cite{drent2022real}. We remark that although data pooling has only value when $N>1$, the analysis in this paper also holds for $N=1$. The set of all systems is denoted by $\mathcal{N}$, i.e. $\mathcal{N} = \{1, \ldots,N\}$. We assume that each system has a critical component such that the system breaks down whenever this component fails. The deterioration processes of these components are modeled as independent Poisson processes with the same rate $\lambda$, denoted with $\{X_i(t), t\geq 0 \}$, with $X_i(0) = 0$, for $i \in \mathcal{N}$. A component of system $i \in \mathcal{N}$ deteriorates until its deterioration level reaches or crosses a finite deterministic failure threshold, denoted with $\xi_i \in \N_+$, where  $\N_+ \triangleq \{1,2,\ldots\}$, after which the component has failed. This failure threshold $\xi_i$ is essentially the maximum physical capacity of a component to withstand the accumulated damage and under which system $i\in \mathcal{N}$ still adequately performs its function. In most practical situations, components of the same type will have the same failure threshold $\xi_i$. However, to allow for the setting in which components have different capacities to withstand deterioration -- which is reasonable when some components are of better quality than others -- we let the failure threshold $\xi_i$ depend on system $i\in \mathcal{N}$. Observe that the same failure threshold for each component is of course a special case of this general set-up.

The deterioration levels are monitored at equally spaced decision epochs, though failure moments can happen at any point in time (i.e. not only at decision epochs). Replacing only at decision epochs is a reasonable assumption given that critical components in these systems typically have mean lifetimes ranging from 1 to 10 years, while maintenance decisions are made much more frequently, often on a daily to weekly basis \citep{oner2013redundancy,lamghari2021new}. For convenience, we rescale time such that the time between two decision epochs equals 1. If at a decision epoch a component of system $i \in \mathcal{N}$ is failed, it needs to be replaced correctively at costs $c^i_u>0$. Such a failure can be prevented by performing a preventive replacement, which costs $c^i_p>0$, with $c^i_p<c^i_u$ for all $i \in \mathcal{N}$. 
Corrective maintenance is more expensive because it includes costs caused by a component failure in addition to the costs related to the replacement (e.g. unplanned downtime costs). Both replacements involve a new component that starts deteriorating again from level 0 according to a Poisson process with rate $\lambda$, that is, $\{X_i(t), t\geq 0 \}$ is reset to $X_i(0) = 0$. We assume that replacement times are negligible. This is a reasonable assumption given the efficiency of replacing old components with new ones, which usually takes only a few minutes to an hour -- significantly shorter than the time between consecutive decision epochs.

The systems have a common, finite lifespan, with length $T\in\N_+$ time units. This lifespan represents the time from their introduction until they are taken out of service, with typical durations ranging from 10 to 30 years \citep{oner2013redundancy}. We let this lifespan consist of $T$ discrete time steps corresponding to the intervals between consecutive decision epochs.  

The maintenance service provider, responsible for maintaining the set of $N$ systems, seeks to minimize the total  expected maintenance costs -- due to both corrective and preventive replacements of components -- over the systems' lifespan. In dealing with this optimization problem, the maintenance service provider faces another layer of uncertainty in addition to the random shock arrivals. That is, the components used for all replacements always have the same rate but this rate is a-priori unknown and needs to be inferred based on the observations of the deterioration processes throughout their lifespan. Since all components have the same rate, the maintenance service provider can pool and utilize all accumulated data together in real-time when inferring this unknown rate.   
To this end, we adopt a Bayesian approach and treat the unknown rate $\lambda$ as a random variable denoted with $\Lambda$. Upon the start of operating all systems, at $t=0$,  $\Lambda$ is modeled by a Gamma distribution with shape parameter $\alpha_0$ and rate parameter $\beta_0$. The subscript notation reflects that this corresponds to $t=0$; we adopt this notation in the remainder of this paper. Thus, at $t=0$, the density function of  $\Lambda$ is equal to
\begin{align*}
f_{\Lambda}(\lambda;\alpha_0,\beta_0) & = \frac{ \lambda^{\alpha_0-1} e^{-\beta_0 \lambda} \beta_0^{\alpha_0}}{\Gamma(\alpha_0)} \quad \text{ for } \lambda > 0, \quad \alpha_0, \beta_0 > 0,
\end{align*} 
where $\Gamma{(\cdot)}$ denotes the Gamma function. Estimation procedures are available in the literature for obtaining the parameters of this initial belief based on expert knowledge or historical data \citep[see, e.g.,][]{aronis2004inventory, drent2022real}. Suppose that at decision epoch $t \in \N_+$, we observed a cumulative amount of $k$ deterioration increments from all installed components. As degradation is modeled by a Poisson process, which is a non-decreasing, integer-valued process, we know that the degradation increments are non-negative and integer-valued as well. Hence, we know that the cumulative sum of all deterioration increments from all installed components, $k$, will always be a non-negative integer. Our choice for the Gamma distribution is not only empirically grounded \citep[e.g.][]{aronis2004inventory,drent2022real}, but also mathematically convenient and therefore quite customary in the literature. Indeed, it is well-known that the Gamma distribution is a conjugate prior for the Poisson distribution, which implies that the new posterior distribution describing our belief of $\Lambda$ is again a Gamma distribution but with updated parameters \citep[see, e.g.,][Chapter 2]{gelman1995bayesian}: 
\begin{align}\label{updatingChap5}
\alpha_t = \alpha_0 + k \quad \mbox{and }\quad \beta_t = \beta_0 +  N\cdot t.
\end{align}
Observe that from the updating scheme in Equation \eqref{updatingChap5}, it is immediately clear that the data stemming from all $N$ systems is pooled for learning the unknown rate $\lambda$ that the systems have in common. In Bayesian terminology, $k$ is the sufficient statistic (which is thus linear in the observations) and $N\cdot t$ is the total amount of observations at decision epoch $t$. At each decision epoch, based on her current belief of $\Lambda$, the maintenance service provider wishes to predict the future evolution of the deterioration of each component so that she can decide upon potential replacements. This prediction is encoded in the posterior predictive distribution. For this Gamma-Poisson model, it is well-known that the posterior predictive distribution is a Negative Binomial distribution \citep[see, e.g.,][]{gelman1995bayesian}. Specifically, given parameters $\alpha_t$ and $\beta_t$, the deterioration increment (i.e. $X_i(t+1) - x_i(t)$ with $x_i(t)$ the current deterioration at decision epoch $t$) of a component at system $i$ at the next decision epoch, denoted with $Z_i$, is Negative Binomially distributed with parameters 
\begin{align}\label{negativebinom}
r = \alpha_t \quad \mbox{and }\quad p = \frac{\beta_t}{\beta_t + 1}, 
\end{align}
where $r$ is the number of successes and $p$ is the success probability, so that $Z_i$ can be interpreted as the number of failures until the $r^{th}$ success. In the remainder we use the notation $Z\sim NB(r,p)$ to denote that $Z$ is a Negative Binomially distributed random variable with parameters $r$ and $p$.  

Equation \eqref{negativebinom} together with the updating scheme in \eqref{updatingChap5} can be used to construct an updated posterior predictive distribution at each decision epoch of the next deterioration increments in real-time based on the observed data.
Since the posterior predictive distributions of the deterioration increments of each system are fully described by only the current decision epoch $t$ and cumulative amount of deterioration increments $k$, it is a Markov process. This allows us to formulate the optimization problem as a finite horizon (with length $T$) MDP equipped with the state variable $k$ for Bayesian inference of the unknown rate. Before doing so, we end this section with an important result that establishes a stochastic ordering property of the posterior predictive distribution $Z$ (for brevity we drop the dependence on $k$, $N$ and $t$) in the cumulative amount of deterioration increments $k$ when everything else is fixed.
\begin{lemma}\label{increasingStochastic} The posterior predictive random variable $Z$ is stochastically increasing in $k$ in the usual stochastic order. 
\end{lemma}
\begin{proof}
See \ref{proofLemma1}. 
\end{proof}
Lemma \ref{increasingStochastic} implies that if the sum of observed deterioration increments increases, and all else is fixed, then the next random deterioration increments are more likely to take on higher values. This is also intuitive since the mean  increment ($\frac{\alpha_t}{\beta_t}$) increases in $k$, see Equation \eqref{updatingChap5}.  

\section{Markov decision process formulation}\label{poolingMDP}
We will now formulate the problem described in the previous section as an MDP. 
The state space of the MDP is represented by the set $\mathcal{S}\triangleq \N^{N+1}_{0}$ where $\N_{0} \triangleq \N_+ \cup \{0\}$ represents the set of non-negative integers. For a given state $(\bm{x},k)\in\mathcal{S}$, $\bm{x}=(x_1,x_2,\ldots, x_N)$ represents the vector of all deterioration levels, and $k$ denotes the sum of all deterioration increments. Recall that as we are dealing with Poisson degradation, both the deterioration levels and the sum of deterioration increments are non-negative integer-valued. For a given state $(\bm{x},k)\in\mathcal{S}$, let $\mathcal{A}(\bm{x})$ denote the action space. For any action $\bm{a} =(a_1, a_2, \ldots, a_N) \in \mathcal{A}(\bm{x})$, $a_i$ represents the action per system, with $a_i \in \{0,1\}$ when $ x_i < \xi_i$ and  $a_i = 1$ otherwise. 
Here, $a_i=0$ corresponds to taking no action and $a_i=1$ corresponds to performing maintenance on the component of system $i$, respectively. This implies that if the critical component of system $i$ is failed (i.e. $x_i \geq \xi_i$), then the maintenance service provider must (correctively) replace it. For all components that have not failed, the maintenance service provider can choose to either preventively replace it, or do nothing and continue to the next decision epoch.  

Given the state $(\bm{x},k)\in\mathcal{S}$ and an action $\bm{a} = (a_1, a_2, \ldots, a_N) \in \mathcal{A}(\bm{x})$, the maintenance service provider incurs a direct cost, denoted by $C(\bm{x},\bm{a})$, equal to 
\begin{align}\label{directcost}
C(\bm{x},\bm{a}) \triangleq \sum_{i\in\mathcal{N}} \left( a_i \big(1-\mathbb{I}_i(\bm{x})\big)c^i_p + \mathbb{I}_i(\bm{x})c^i_u \right),
\end{align}
where $\mathbb{I}_i(\bm{x})$ is an indicator function that indicates whether the component of system $i$ has failed in the deterioration vector $\bm{x}$; that is, $\mathbb{I}_i(\bm{x}) = 0$ if $x_i < \xi_i$ and $\mathbb{I}_i(\bm{x}) = 1$ otherwise. 

Let $V^N_t(\bm{x},k)$ denote the optimal expected total cost over decision epochs $t, t+1, \ldots, T$, starting from state $(\bm{x},k)\in\mathcal{S}$, and let the terminal cost, $V^N_T(\bm{x},k)$, be equal to the function $C(\bm{x}) \triangleq \sum_{i\in\mathcal{N}}\mathbb{I}_i(\bm{x})c^i_u$ for all $k$. This terminal cost function essentially assigns a corrective maintenance cost to failed components, while no costs are incurred for non-failed components. Then, by the principle of optimality,   $V^N_t(\bm{x},k)$ satisfies the following recursive Bellman optimality equations
\begin{align}\label{optimality}
V^N_t(\bm{x},k) = \min_{\bm{a} \in \mathcal{A}(\bm{x})} \Bigg\{C(\bm{x},\bm{a}) + \E_{\bm{Z}}\bigg[  V^N_{t+1}\Big(\bm{x'}+\bm{Z},k+\sum_{i\in \mathcal{N}}Z_i \Big) \bigg]  \Bigg\}, 
\end{align}
where $\bm{Z} = (Z_1,Z_2,\ldots, Z_N)$ is an $N$-dimensional random vector with $Z_i \sim  NB\Big(\alpha_0 + k, \frac{\beta_0+N\cdot t}{\beta_0 + N\cdot t + 1}\Big)$ (all $Z_i$'s are independent and identically distributed), $\E_{\bm{Z}}$ denotes that the expectation is taken with respect to $\bm{Z}$, and $\bm{x'} =\big(x'_1,x'_2,\ldots, x'_N\big)$  with 
\begin{align} \label{newinstalled}	
x'_i = \begin{cases}
  x_i  &  \text{if }a_i =0, \\
  0 & \text{if }a_i = 1.
\end{cases}
\end{align}
We also refer to $V^N_t(\bm{x},k)$ as the value function of the original MDP. The first part between the brackets is the direct costs while the second part is the expected future costs of taking action $\bm{a}$ in state $(\bm{x},k)$. Specifically, each component's deterioration accumulates further according to the posterior predictive distribution that corresponds to state $(\bm{x},k)$, and $k$ increases with the sum of all those increments. Systems that are maintained start with an as-good-as-new component, which is governed by the auxiliary vector $\bm{x'}$ which ensures that $x'_i=0$ when $a_i=1$ (see Equation \eqref{newinstalled}).  The formulation in \eqref{optimality} shows that the learning process about the unknown rate $\lambda$ is pooled through the evolution of the common state variable $k$, while the future evolution of all individual deterioration processes depends on all pooled information and the parameter $N$.
The existence of an optimal policy in this setting is guaranteed, see e.g., Proposition 3.4 of \cite{bertsekas1978stochastic}. 

Observe that the minimum total expected cost for $N$ systems over the complete lifespan of length $T$ is given by $V^N_{0}(\bm{0},0)$ ($\bm{0}$ denotes the N-dimensional zero vector) which can be found by solving Equation \eqref{optimality} via backward induction. It is however clear from the formulation in \eqref{optimality}, that as the number of systems grows, the problem will increasingly suffer from the curse of dimensionality: The cardinality of both the action and state space grow exponentially in $N$. 
Instead of solving \eqref{optimality} (referred to as the original MDP) directly, we will therefore construct an alternative MDP and show that the original MDP can be decomposed into $N$ of these alternative MDPs: One for each system $i\in \mathcal{N}$. This decomposition is imperative as it allows us to (i) analyze the benefits of pooling of learning when $N$ is relatively large without suffering from the curse of dimensionality, and (ii) establish structural properties of the optimal policy. 

To this end, let $\tilde{V}^{N,i}_t(x,k)$ denote the optimal expected total cost  for system $i\in \mathcal{N}$, over decision epochs $t, t+1, \ldots, T$, starting from state $(x,k)\in\N^2_{0}$, and let the terminal cost, $\tilde{V}^{N,i}_T(x,k)$, be equal to the function $C_i(x) \triangleq \mathbb{I}_i(x)c^i_u$ for all $k$. Then,  $\tilde{V}^{N,i}_t(x,k)$ satisfies the following recursive Bellman optimality equations
\begin{align}\label{decomposedMPD}
\tilde{V}^{N,i}_t(x,k) = \min_{a \in \mathcal{A}(x)}\Bigg\{C_i(x,a) + \E_{(Z,K)}\bigg[\tilde{V}^{N,i}_{t+1}\Big(x\cdot(1-a)+Z,k+Z+K\Big) \bigg]  \Bigg\}, 
\end{align}
where  $Z \sim NB\Big(\alpha_0 + k, \frac{\beta_0+N\cdot t}{\beta_0 + N\cdot t + 1}\Big)$, $K \sim NB\Big( (N-1)\cdot(\alpha_0 + k), \frac{\beta_0+N\cdot t}{\beta_0 + N\cdot t + 1}\Big)$, $\E_{(Z,K)}$ denotes that the expection is taken with respect to $Z$ and $K$,  and
\begin{align}\label{directcostalternative}
C_i(x,a) \triangleq a \big(1-\mathbb{I}_i(x)\big)c^i_p +  \mathbb{I}_i(x)c^i_u.
\end{align}
The indicator functions and actions (spaces) are as defined before. It is noteworthy to mention that the formulation in \eqref{decomposedMPD} in fact resembles a single component optimization problem in isolation, where the transition probabilities depend on both parameter $N$ and the state variable $k$. The evolution of the state variable $k$ depends on the random deterioration increment of the component ($Z$) but it also accounts for the evolution of the other components through the random variable $K$. Below we present the decomposition result, which establishes that the value function of the original MDP is the sum of all $N$ value functions of the alternative MDPs.

\begin{theorem}\label{decomposition}
For each $t \in \{0,1,\ldots, T\}$, we have
$$V^N_t(\bm{x},k) = \sum_{i\in\mathcal{N}} \tilde{V}^{N,i}_t(x_i,k), \quad \mbox{for all } (\bm{x},k)\in\mathcal{S}.$$
\end{theorem}
\begin{proof}
We prove the statement using induction on $t$. Note that the terminal values can be decomposed: $$V^N_T(\bm{x},k) = C(\bm{x}) = \sum_{i\in\mathcal{N}}\mathbb{I}_i(\bm{x})c^i_u = \sum_{i\in \mathcal{N}} \mathbb{I}_i(x_i)c^i_u = \sum_{i\in \mathcal{N}} \tilde{V}^{N,i}_T(x_i,k),$$ 
so that the statement trivially holds for the base case $T$.  Assume that the statement holds for some $t+1, 0<t+1\leq T$, we will show that the statement then also holds for $t$. 
We have, by Equation \eqref{optimality}, 
\begin{align}
V^N_t(\bm{x},k) &= \min_{\bm{a} \in \mathcal{A}(\bm{x})} \Bigg\{C(\bm{x},\bm{a}) + \E_{\bm{Z}}\bigg[  V^N_{t+1}\Big(\bm{x'}+\bm{Z},k+\sum_{j\in\mathcal{N}} Z_j \Big) \bigg]  \Bigg\} \nonumber \\
&\stackrel{(a)}{=}  \min_{\bm{a} \in \mathcal{A}(\bm{x})} \Bigg\{ \sum_{i\in\mathcal{N}} C_i(x_i,a_i) + \E_{\bm{Z}}\bigg[  V^N_{t+1}\Big(\bm{x'}+\bm{Z},k+\sum_{j\in\mathcal{N}} Z_j \Big) \bigg] \Bigg\} \nonumber \\
&\stackrel{(b)}{=} \min_{\bm{a} \in \mathcal{A}(\bm{x})} \Bigg\{ \sum_{i\in\mathcal{N}} C_i(x_i,a_i) + \E_{\bm{Z}}\bigg[ \sum_{i\in \mathcal{N}}  \tilde{V}^{N,i}_{t+1}\Big(x'_i+Z_i , k+\sum_{j\in\mathcal{N}} Z_j \Big) \bigg]\Bigg\}  \nonumber \\
&\stackrel{(c)}{=}  \min_{\bm{a} \in \mathcal{A}(\bm{x})} \Bigg\{ \sum_{i\in\mathcal{N}} C_i(x_i,a_i) + \sum_{i\in\mathcal{N}} \E_{\bm{Z}}\bigg[   \tilde{V}^{N,i}_{t+1}\Big(x'_i+Z_i , k+Z_i+\sum_{j\in\mathcal{N}\setminus \{i\}} Z_j \Big) \bigg]\Bigg\}  \nonumber \\
&\stackrel{(d)}{=} \sum_{i\in\mathcal{N}} \min_{a_i \in \mathcal{A}(x_i)} \Bigg\{ C_i(x_i,a_i) +  \E_{(Z_i,K)} \bigg[   \tilde{V}^{N,i}_{t+1}\Big(x_i(1-a_i)+Z_i, k+Z_i+K \Big) \bigg]\Bigg\}  \nonumber \\
&= \sum_{i\in\mathcal{N}} \tilde{V}^{N,i}_t(x_i,k), \nonumber 
\end{align}
where $(a)$ is because the direct costs can be decomposed (see \eqref{directcost} and \eqref{directcostalternative}), $(b)$ is due to the induction hypothesis, $(c)$ is because of the linearity of an expectation and extracting $Z_i$ from the summation, $(d)$ is because the sum of $N-1$ independent Negative Binomially distributed random variables with $r=\alpha_0+k$ and $p=\frac{\beta_0+N\cdot t}{\beta_0 + N\cdot t + 1}$ is again Negative Binomially distributed with the same $p$ but with $r=(N-1)\cdot(\alpha_0+k)$ \citep[see, e.g.,][Chapter 6]{dasgupta2010fundamentals}, and the last equality follows from using Equation \eqref{decomposedMPD}. 
\end{proof}

The decomposition in Theorem \ref{decomposition} reduces the computational burden of solving \eqref{optimality} significantly. It collapses the original, high-dimensional MDP into $N$ 2-dimensional MDPs with a binary action space, each with their own cost structure and failure threshold, while still taking into account pooled learning across the $N$ systems. Next to reducing this computational burden, the decomposition result also eases the process of establishing structural properties on a system level, which is the topic of the next section.

Next to the trivial conditions that the action space and cost function should be decomposable and that actions should not influence the future evolution of the deterioration processes, there are two essential conditions required for our decomposition result to hold. As our proof relies crucially on these two general conditions, we believe they can guide future research on pooled learning in MDPs. Therefore, we discuss these conditions in the remark below.

\begin{remark} The decomposition result can be applied to other high-dimensional Bayesian MDPs in which data is pooled to learn a common but unknown parameter. Specifically, the decomposition result necessitates two conditions related to the underlying conjugate pair of the MDP. First, the sufficient statistic in the conjugate pair for learning the parameter should be linear in the observations. This condition enables step $(c)$ in the proof where we extract $Z_i$ from the summation. Secondly, the resulting posterior predictive distribution should be closed under convolutions. This enables step $(d)$ in the proof where we use the fact that the convolution of $N-1$ Negative Binomially distributed random variables is again Negative Binomially distributed. One other conjugate pair -- next to the Gamma - Poisson pair used in this paper -- that, for instance, satisfies these conditions and which is used very often in the OM literature is the Normal - Normal pair. This pair is generally adopted when the mean of a Normal distribution with known variance is unknown and needs to be learned.
\end{remark}

\section{Structural properties} \label{structAlternative}
In this section we establish structural properties of the alternative MDP, which then carry over to the original MDP through our decomposition result. We first derive monotonicity properties of the value function, and then use these properties to establish the optimality of a control limit policy. We finally show that the control limit approaches the failure level as the pooled data increases. To that end, we first rewrite \eqref{decomposedMPD} into the conventional formulation for single component optimization problems:
\begin{align}\label{eq:decomposedAlt}
&\tilde{V}^{N,i}_t(x,k) = \quad &  \nonumber \\
&\quad \begin{cases}
    c^i_u +  \E_{(Z,K)}\bigg[\tilde{V}^{N,i}_{t+1}\Big(Z,k+Z+K\Big) \bigg], & \text{if } x \geq \xi_i, \\
 \mbox{min}\left\{c^i_p +\E_{(Z,K)}\bigg[\tilde{V}^{N,i}_{t+1}\Big(Z,k+Z+K\Big) \bigg]; \E_{(Z,K)}\bigg[\tilde{V}^{N,i}_{t+1}\Big(x+Z,k+Z+K\Big) \bigg] \right\}, & \text{if } x < \xi_i. \\
  \end{cases}&
\end{align}
The first case in Equation \eqref{eq:decomposedAlt} is because failed components must be replaced correctively at cost $c^i_u$. If the component’s deterioration level is less than $\xi_i$, we can either perform a preventive
replacement, which costs $c^i_p$, or leave the component in operation until the next decision epoch at no cost. The terminal costs are as introduced before. The next result establishes the monotonicity of the value function $\tilde{V}^{N,i}_t(x,k)$ in $x$ and $k$.
\begin{proposition}\label{valuemonotoneX}
For each $t\in\{0,1,\ldots,T\}$ and $i\in \mathcal{N}$, the value function $\tilde{V}^{N,i}_t(x,k)$ is:
\begin{enumerate}[label=(\roman*)]
\item non-decreasing in $x$, and
\item non-decreasing in $k$.
\end{enumerate}
\end{proposition}
\begin{proof}
See \ref{proofOfPropValueCBM}.
\end{proof}
Proposition \ref{valuemonotoneX} implies that (i) if a component is more deteriorated or (ii) when the total amount of deterioration increments is higher, we expect to incur higher costs. This is intuitive: A higher level of deterioration increases the probability of a costly failure and/or the need for preventive replacement, while a higher total amount of deterioration increments implies that all components are deteriorating relatively fast (i.e. $\lambda$ is larger). 

By Proposition \ref{decomposition}, we also conclude that the value function $V^N_t(\bm{x},k)$ of the original MDP is non-decreasing in the standard component-wise order in $\bm{x}$, and non-decreasing in $k$. The former means that for any deterioration vectors $\bm{x}$ and $\bm{x'}$ such that $x_i\leq x'_i$ for all $i\in \mathcal{N}$, we have that $V^N_t(\bm{x},k) \leq V^N_t(\bm{x'},k)$. The intuition behind this is similar to the intuition behind Proposition \ref{valuemonotoneX}.

The next result establishes the optimality of a control limit policy for the alternative MDP that depends on the state variable $k$. 

\begin{proposition}\label{controllimit}
For each $t\in\{0,1,\ldots,T-1\}$, $k\in\N_{0}$, and $i\in \mathcal{N}$, there exists a control limit $\delta_i^{(k,t)}$, $0<\delta_i^{(k,t)} \leq \xi_i$, such that the optimal action is to carry out a replacement if and only if $x\geq\delta_i^{(k,t)}$.  
\end{proposition}
\begin{proof}
See \ref{proofCL}. 
\end{proof}

Proposition \ref{controllimit} shows that the control limit at each time of each component, $\delta_i^{(k,t)}$, depends in real-time on the shared learning process across all components through pooling data via the state variable $k$. The optimality of a control limit policy itself is not only intuitive and convenient for the implementation of this optimal policy in practice, it can also be exploited to further decrease the computational burden of solving the original MDP. That is, existing algorithms that rely on these structural properties such as the modified policy iteration algorithm \citep[see][Section 6.5]{puterman2014markov} can be used to efficiently solve the alternative MDP, and hence the original MDP. 

Conceivably, one would expect that as we learn from the pooled data that $\lambda$ is larger (through a higher $k$) and everything else is fixed, we would impose a lower control limit per component, i.e. $\delta_i^{(k,t)}$ is non-increasing in $k$. The intuition behind this is that because it is more likely that deterioration increments will take on higher values (see Lemma \ref{increasingStochastic}), we should replace a component earlier. Although such a non-increasing $\delta_i^{(k,t)}$  would indeed be intuitive, we found numerically that this is in general not true. Specifically, we found that the control limit usually decreases in $k$ first, as expected, but that it always increases eventually to $\xi_i$ as $k$ grows large, in which case it is never optimal to do preventive maintenance. See Figure \ref{asym} for an excellent illustration of this behavior. In this figure, we use, $N=2$, $T=50$, $\alpha_0=4$, and $\beta_0=4$,  and use the same values for the parameters of both components: $c_p=1$, $c_u=10$, and $\xi=10$, and plot the optimal control limit at three decision epochs (10, 25, and 40) as a function of $k$.

\begin{figure}[htb!]
\begin{center}
\includegraphics[width=0.85\textwidth]{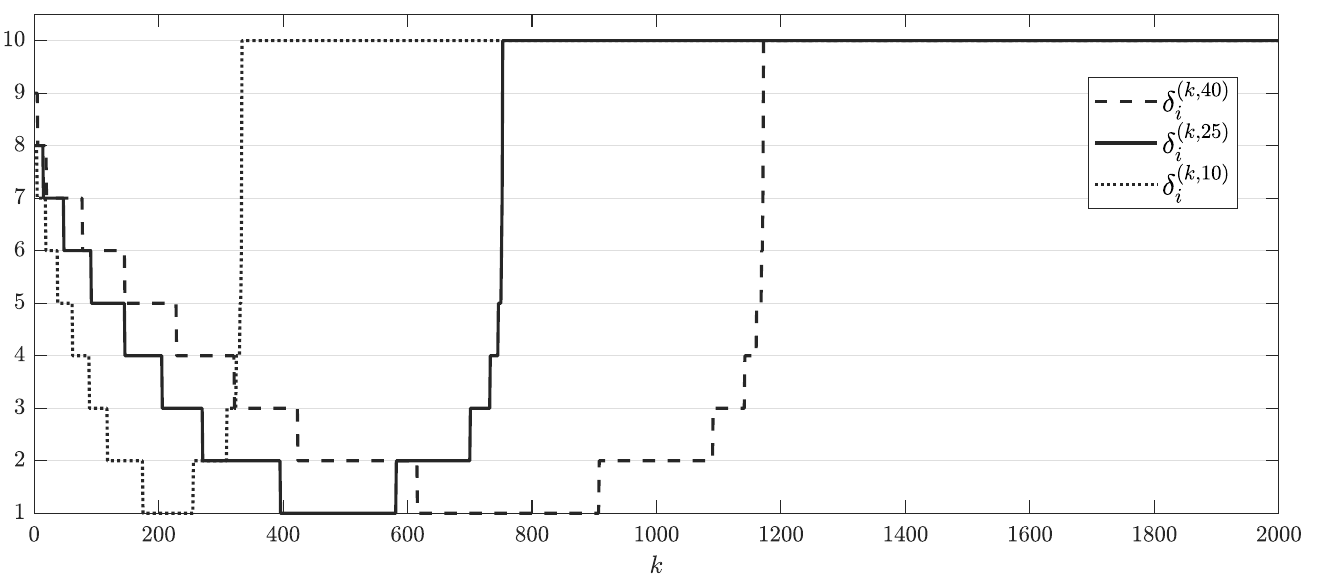}
\caption{The optimal control limit (value on $y$-axis) as a function of $k$ at three different decision epochs. The optimal control limit first decreases in $k$, but then increases to $\xi$ as $k$ grows large.}
\label{asym}
\end{center}
\end{figure}

This limiting behavior, which breaks the monotonic behavior of $\delta_i^{(k,t)}$ in $k$, is formalized in the proposition below. 

\begin{proposition}\label{controllimitAsymptote}
For each $t\in\{0,1,\ldots,T-1\}$ and $i\in \mathcal{N}$, we have $\lim\limits_{k\rightarrow \infty} \delta_i^{(k,t)} = \xi_i$. 
\end{proposition}
\begin{proof}
See \ref{proofk}. 
\end{proof}
While Proposition \ref{controllimitAsymptote} may initially appear counterintuitive, it can be heuristically explained as follows. When $k$ grows very large and everything else is kept fixed, the random deterioration per period grows so large that any component will fail with certainty solely due to the one-period deterioration. In this case, if a component is still working at a certain decision epoch and has deterioration level $x$ $(<\xi_i)$, performing preventive maintenance will induce an extra cost of $c^i_p$ because in the next period, the component will fail anyway, regardless of the value of $x$. To make this argument more explicit, consider a component with deterioration level $x<\xi_i$ at decision epoch $T-1$. As we only have the terminal cost at time $T$, we can readily obtain an expression for the optimality equation from the recursive Bellman equations in \eqref{eq:decomposedAlt}:
\begin{align}\label{eq:decomposedAltAsymtpote}
\tilde{V}^{N,i}_{T-1}(x,k) = \mbox{min}\Biggl\{ \underbrace{c^i_p +\P\big[Z\geq \xi_i\big] \cdot c^i_u + \left(1-\P\big[Z\geq \xi_i\big]\right)\cdot c^i_p}_{\text{preventive maintenance}}\ ;\ \underbrace{\P\big[Z\geq \xi_i-x\big] \cdot c^i_u + \left(1-\P\big[Z\geq \xi_i-x\big]\right)\cdot c^i_p}_{\text{leave in operation}}\Biggl\} .
\end{align}
It is clear that preventive maintenance in this state is not optimal if and only if the following holds: 
\begin{align}\label{eq:decomposedAltAsymtpoteEqn}
c^i_p +\P\big[Z\geq \xi_i\big] \cdot c^i_u + \left(1-\P\big[Z\geq \xi_i\big]\right)\cdot c^i_p > \P\big[Z\geq \xi_i-x\big] \cdot c^i_u + \left(1-\P\big[Z\geq \xi_i-x\big]\right)\cdot c^i_p.
\end{align}
Since the random variable $Z$ is increasing in $k$ (see Lemma \ref{increasingStochastic}), one can show that $\P\big[Z\geq \xi_i\big]\rightarrow 1$ and $\P\big[Z\geq \xi_i-x\big]\rightarrow 1$ for each $x<\xi_i$ as $k\rightarrow \infty$. This implies, using \eqref{eq:decomposedAltAsymtpote}, that at decision epoch $T-1$ as $k\rightarrow \infty$, leaving the component in operation costs $c_u^i$, while performing preventive maintenance costs $c^i_p +c_u^i$ (an extra cost of $c^i_p$). Since $c^i_p>0$, Equation \eqref{eq:decomposedAltAsymtpoteEqn} will always hold for any $x<\xi_i$ at decision epoch $T-1$ as  $k\rightarrow \infty$. In the proof of Proposition \ref{controllimitAsymptote} in \ref{proofk}, we formalize this heuristic argument and do so for each decision epoch $t\in\{0,1,\ldots,T-1\}$.

\section{Numerical study}\label{PoolingNum}
This section reports the results of a comprehensive numerical study in which we assess the benefits of pooling data to learn an a-priori unknown parameter, while optimally maintaining $N$ systems over a finite lifespan $T$. Although the results in the previous sections hold for asymmetric -- in terms of costs and failure thresholds -- systems, we shall focus on symmetric systems in this numerical study. By doing so, the value function $\tilde{V}^{N}_0(0,0)$ (we drop the index $i$ as we consider symmetric systems) gives us the cost per system over its lifespan when the data of $N$ systems is pooled. We can use this cost per system to assess the value of pooling learning as a function of $N$ compared to not pooling. To this end, we define the performance measure $\Delta = 100 \left[1-\frac{\tilde{V}^{N}_0(0,0)}{\tilde{V}^{1}_0(0,0)}\right],$
which is the percentage savings per system over the lifespan when the learning of $N$ systems is pooled compared to not pooling any data for those systems and learning the unknown rate independently from the other systems. 

We first perform an extensive numerical study. Recall that the initial parameter uncertainty is modeled by the random variable  $\Lambda$ which has a Gamma distribution with shape $\alpha_0$ and rate $\beta_0$.  By fixing the mean of $\Lambda$ and subsequently varying its coefficient of variation, we can thus increase or decrease the initial parameter uncertainty. We do so by solving the following set of equations for $\alpha_0$ and $\beta_0$: 
\begin{align*}
 \E[\Lambda] = \frac{\alpha_0}{\beta_0}, \mbox{ and }  cv_{\Lambda} = \frac{1}{\sqrt{\alpha_0}},
\end{align*}
where $cv_{\Lambda}$ is the coefficient of variation of $\Lambda$. This allows us to explicitly study the impact of the uncertainty (in terms of its mean and coefficient of variation) on the pooling effects. Our testbed consists of 2268 instances. These are obtained by permuting all parameter values in Table \ref{tab:testbedChap5}, with the corrective maintenance cost $c_u$ held fixed at 10. These values are representative for the capital goods industry and are in line with the maintenance literature (see, e.g., \citep[][]{van2017maintenance} on typical maintenance costs, and \citep[][]{drent2022real} on initial parameter uncertainty). For each instance of the test bed we compute the relative savings $\% \Delta$. 

\begin{table*}[!htbp]
  \centering
              \fontsize{8pt}{8pt}\selectfont
  \caption{{Input parameter values for numerical study.}}
  \label{tab:testbedChap5}
  \begin{tabularx}{0.6\textwidth}{c X c l}
    \toprule
    & Input parameter& No. of choices&Values\\
                             \midrule
1 & Number of systems, $N$ & 7 &  1, 2, 4, 6, 8, 10, 20      \\
2 & Failure threshold, $\xi$ & 2 & 7, 10\\
3 & Length of lifespan, $T$							        & 3 &  50, 70, 90 \\
4 & Preventive maintenance cost, $c_p$							        & 3 &  0.5, 1, 1.5 \\
5 & Mean of $\Lambda$, $\E[\Lambda]$							        & 3 & 0.5, 0.75, 1 \\
6 & Coefficient of variation of $\Lambda$, $cv_{\Lambda}$						        & 6 & 0.1, 0.25, 0.5, 1, 2, 4  \\
   \bottomrule
  \end{tabularx}
\end{table*}

The results of the numerical study are summarized in Table \ref{tab:resultsTestbedChap5}. In this table, we present the average and maximum relative savings $\% \Delta$. For each value of $N$, we first present the average relative savings for subsets of instances with the same value for a given input parameter of Table \ref{tab:testbedChap5} (row wise), and then present the average results for all instances with that fixed value of $N$ (bottom row), where each average value is accompanied with the maximum value in brackets. 

\begin{table}[!htbp]
\centering
\setlength{\tabcolsep}{4pt}
\caption{{Relative savings ($\%\Delta$) due to pooled learning.}}
\fontsize{8pt}{8pt}\selectfont
\label{tab:resultsTestbedChap5}
\begin{tabular}{lllccccccc}
\toprule
Input &       &       & \multicolumn{6}{c}{$N$} & \\
\cmidrule{4-9}  
parameter & \multicolumn{1}{l}{Value} &       & 2  & 4  & 6 & 8 & 10 & 20  \\
\midrule
\multirow{2}[2]{*}{$\xi$} & \multicolumn{1}{l}{7}                                                            &         &	2.7	(16.4)	&	9.8	(37.2)	&	14.8	(59.2)	&	18.0	(71.2)	&	19.7	(77.2)	&	24.5	(88.5)	\\
& \multicolumn{1}{l}{10}                                                                                                      &    &     3.4	(19.7)	&	9.9	(36.8)	&	14.5	(59.4)	&	17.1	(71.8)	&	19.1	(79.7)	&	22.7	(89.2)	\\
\addlinespace
\multirow{3}[2]{*}{$T$} & \multicolumn{1}{l}{50}                                                            &         & 3.0	(19.7)	&	9.7	(36.3)	&	14.5	(57.3)	&	17.4	(70.2)	&	19.3	(78.4)	&	23.6	(88.4)	\\
& \multicolumn{1}{l}{70}                                                                                                     &         & 3.1	(19.7)	&	9.8	(36.6)	&	14.7	(58.6)	&	17.6	(71.2)	&	19.4	(79.2)	&	23.6	(88.9)	\\
& \multicolumn{1}{l}{90}                                                                                                       &         & 3.1	(19.7)	&	10.0	(37.2)	&	14.8	(59.4)	&	17.6	(71.8)	&	19.5	(79.7)	&	23.6	(89.2)	\\
\addlinespace
\multirow{3}[1]{*}{$c_p$ } & \multicolumn{1}{l}{0.5}                                               &         & 4.4	(19.7)	&	12.7	(37.2)	&	18.3	(59.4)	&	21.6	(71.8)	&	23.7	(79.7)	&	28.0	(89.2)	\\
& \multicolumn{1}{l}{1}                                                                                                      &         & 2.9	(12.7)	&	9.5	(33.5)	&	14.2	(53.9)	&	17.1	(65.1)	&	18.9	(72.5)	&	22.9	(82.8)	\\
& \multicolumn{1}{l}{1.5}                                                                                                      &         & 1.9	(8.7)	&	7.4	(29.9)	&	11.5	(48.9)	&	14.0	(59.4)	&	15.6	(65.8)	&	19.9	(77.6)	\\
\addlinespace
\multirow{3}[1]{*}{$\E[\Lambda]$ } & \multicolumn{1}{l}{0.5}                                               &         & 3.0	(19.7)	&	8.3	(33.5)	&	12.1	(47.5)	&	14.6	(60.1)	&	16.2	(69.1)	&	19.6	(83.6) \\
& \multicolumn{1}{l}{0.75}                                                                                                      &         & 3.0	(18.7)	&	10.0	(36.8)	&	14.9	(54.6)	&	17.8	(67.6)	&	19.7	(76.3)	&	23.8	(87.5)	\\
& \multicolumn{1}{l}{1}                                                                                                     &         & 3.1	(14.3)	&	11.2	(37.2)	&	16.9	(59.4)	&	20.3	(71.8)	&	22.4	(79.7)	&	27.3	(89.2) \\
\addlinespace
\multirow{6}[1]{*}{$cv_{\Lambda}$ } & \multicolumn{1}{l}{0.1}      &                                         &    0.0	(0.1)	&	0.0	(0.2)	&	0.1	(0.2)	&	0.1	(0.3)	&	0.1	(0.3)	&	0.2	(0.4)	\\
& \multicolumn{1}{l}{0.25}                                                                                                     &         & 0.2	(0.5)	&	0.4	(0.8)	&	0.5	(1.0)	&	0.5	(1.1)	&	0.6	(1.2)	&	0.8	(1.4)	\\
& \multicolumn{1}{l}{0.5}                                                                                                     &         & 0.6	(1.3)	&	1.3	(2.7)	&	1.8	(3.5)	&	2.1	(4.1)	&	2.4	(5.3)	&	3.1	(6.4)	\\
& \multicolumn{1}{l}{1}                                                                                                     &         & 4.1	(12.0)	&	7.8	(22.5)	&	9.4	(26.3)	&	10.4	(28.3)	&	11.0	(31.1)	&	12.4	(35.7)	\\
& \multicolumn{1}{l}{2}                       
&         & 10.3	(19.7)	&	22.7	(36.8)	&	29.5	(44.8)	&	33.8	(51.5)	&	36.7	(57.3)	&	46.0	(73.9)	\\
& \multicolumn{1}{l}{4}                                                                                                      &         & 13.0	(24.6)	&	27.2	(37.2)	&	35.3	(59.4)	&	40.8	(71.8)	&	44.9	(79.7)	&	56.9	(89.2)	\\

\midrule
Total &                                                                                                                      &         & 3.6	(24.6)	&	9.8	(37.2)	&	14.0	(59.4)	&	16.5	(71.8)	&	18.2	(79.7)	&	22.3	(89.2)	\\
    \bottomrule
    \end{tabular}%
\end{table}%

Based on the results in Table \ref{tab:resultsTestbedChap5}, we can state the following main observations: 
\begin{enumerate}
\item Pooling of data for learning a common, unknown parameter can lead to significant savings compared to not pooling data and learning it independently.
\item The magnitude of the savings seems to be inextricably linked with the magnitude of uncertainty in the parameter $\lambda$ measured by the coefficient of variation of $\Lambda$. When  $cv_{\Lambda}$ is high,  savings of up to 56.9\% on average (over all instances with $cv_{\Lambda}=4$ and $N=20$)  can be achieved, while if  $cv_{\Lambda}$ is low, savings become almost negligible ($\leq 0.2\%$ on average). This can be explained as follows. When there is high uncertainty in the unknown parameter, pooling data allows the maintenance service provider to faster learn the unknown parameter compared to learning it from data generated by a single system. This result implies that pooling data is especially beneficial for real-life settings where there is high uncertainty in $\lambda$ through limited knowledge, limited historical data, and/or poor estimation procedures.
The opposite is also true. When there is little uncertainty in the unknown parameter, the benefit of data pooling vanishes; a maintenance service provider already has an accurate belief of the unknown parameter that needs little updating. 
\item When comparing the average savings for increasing values of $N$, we find that pooling has already a significant impact for small values of $N$, and that the marginal savings gradually decrease when $N$ increases.
\item The savings for each value of $N$ tend to decrease as the ratio $c_u/c_p$ decreases (recall that we keep $c_u$ fixed and vary $c_p$). When this ratio decreases and $N$ is fixed, maintenance decisions have less impact on the resulting costs -- simply because their cost difference decreases. Consequently, the benefits of utilizing pooled learning in such maintenance decisions also decrease when $c_u/c_p$ decreases. 
\item The savings for each value of $N$ tend to increase  as $\E[\Lambda]$ increases. When $\E[\Lambda]$ increases and $N$ is fixed, the expected deterioration increment between two consecutive decision epochs is larger and, as a result, the optimal control limit will be more conservative. The results suggest that in that regime, the choice of the control limit has a higher impact on the resulting costs than when $\E[\Lambda]$ is low and a less conservative control limit is chosen. By pooled learning, one is able to better choose this control limit, and as a result, the relative savings of pooled learning also increase when $\E[\Lambda]$ increases. 
\end{enumerate}

Observations 1-3 are also illustrated by Figure \ref{poolingEffects}. In this figure we plot the relative savings ($\%\Delta$) as a function of $N$ for various values of $cv_{\Lambda}$ when $\E[\Lambda]= 0.75$, $\xi=10$, $c_p=0.5$, $c_u=10$, and $T=90$.

\begin{figure}[htb!]
\begin{center}
\includegraphics[width=0.8\textwidth]{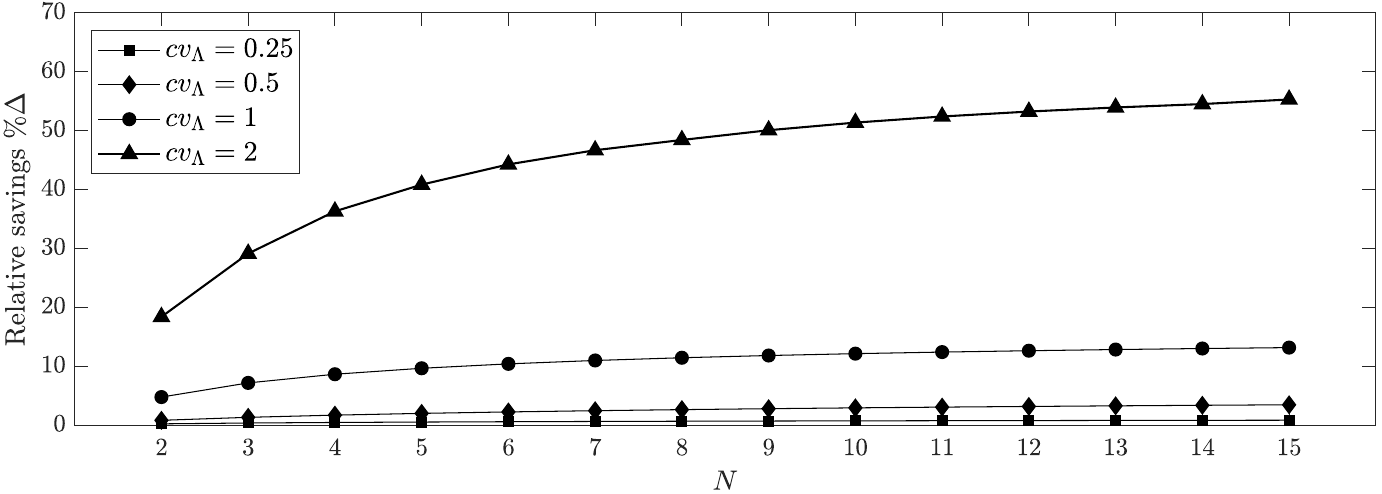}
\caption{{Relative savings ($\%\Delta$) as function of $N$ for various values of $cv_{\Lambda}$ for $\E[\Lambda]=0.75$, $\xi=10$, $c_p=0.5$, $c_u=10$, and $T=90$.}}
\label{poolingEffects}
\end{center}
\end{figure}

The plot indeed shows that for a given level of parameter uncertainty, pooling data across a larger number of systems increases the relative savings. The rate at which the savings increase in $N$ increases significantly in the coefficient of variation. This confirms that pooling data can lead to significant cost reductions, especially when the uncertainty surrounding an unknown parameter is high. 
We further clearly see that the marginal savings due to adding an extra system to the pooled systems decreases as $N$ increases.

\section{An application to spare parts inventory systems}\label{Pooling:spare}
In this section we illustrate that our decomposition result can also be applied to spare parts inventory systems. To make this illustration explicit, we shall redefine some notation introduced in the previous sections.
We consider a maintenance service provider that operates a set of $N\geq 1$ local spare parts warehouses, and we denote this set by $\mathcal{N}=\{1,\ldots,N\}$. Each local warehouse stocks spare parts of the same critical component to serve an installed base of (closely located) technical systems. 
As is common in the spare parts inventory literature \cite[e.g.][]{drent2021expediting}, we assume that that demand for spare parts at each local warehouse follows an independent Poisson process. The rate of these Poisson processes $\lambda$ is identical across all local warehouses. This is a reasonable assumption when the installed bases served by each local warehouse are of similar size. 

The maintenance service provider is concerned with stocking decisions over a finite horizon of $T$ periods. 
At the start of each such period, the maintenance service provider decides how many new spare parts are transported to each local warehouse $i\in\mathcal{N}$. Each unit has a transportation cost $c_v^i$. Since the lead times to the local warehouses are typically much shorter than the duration of a period, we assume that these new spare parts are instantly delivered, after which the period commences. When a component in a technical system fails during the period, local warehouse $i\in\mathcal{N}$ responsible for this system immediately replaces the failed component with a read-for-new one, if it has one available. Otherwise, the part is backordered at unit cost $c_b^i$, which reflects expensive downtimes or emergency shipments from a central depot or an external supplier. Spare parts on stock that are carried over to the next period cost $c_h^i$ per unit. We account for both backorder and holding costs at the end of each period.
We assume that each period lasts 1 time unit so that demand in each period is Poisson distributed with mean $\lambda$. We employ the Bayesian approach of Section \ref{cbmPoolFormulation} to infer the a-priori unknown rate $\lambda$ based on the observed demands at all local warehouses over the entire planning horizon.
Observe that in the updating scheme of this approach (cf. Equation \ref{updatingChap5}), $k$ is now defined as the total cumulative demand at all $N$ local warehouses up to period $t$. Given $k$ and $t$, the posterior predictive $Z_i$ now represents the total demands that arrive at local warehouse $i\in\mathcal{N}$ during the next decision epoch. 

The state space of the Bayesian MDP corresponding to the decision problem described above is given by $\mathcal{S}\triangleq \Z^{N}\times \N_0$.
For a given state $(\bm{x},k)\in\mathcal{S}$, $\bm{x}=(x_1,x_2,\ldots, x_N)$ represents the vector of net inventory levels of all local warehouses before order placement at the start of a period, and $k$ denotes the sum of all observed demands until that period. For a given state $(\bm{x},k)\in\mathcal{S}$, the action space $\mathcal{A}(\bm{x})$ contains all possible net inventory levels after orders are placed and received but before demand is realized, i.e. for any action $\bm{a} = (a_1, a_2, \ldots, a_N) \in \mathcal{A}(\bm{x})$, $a_i \in \{x_i, x_i+1,\ldots \}$ is the net inventory level per local warehouse. As before, we let $\bm{Z} = (Z_1,Z_2,\ldots, Z_N)$ denote an $N$-dimensional random vector with $Z_i \sim  NB\Big(\alpha_0 + k, \frac{\beta_0+N\cdot t}{\beta_0 + N\cdot t + 1}\Big)$. As is customary in inventory theory, the direct cost in a given period accounts for the expected holding and backorder costs of the orders placed in that period. As such, the total transportation, holding, and backorder costs over all local spare parts warehouses, denoted $C(\bm{a},\bm{x},k)$, is given by
\begin{align*}
   C(\bm{a},\bm{x},k) \triangleq  \sum_{i\in\mathcal{N}} \left( c_v^i(a_i-x_i)+c^i_h\E\left[(a_i - Z_i)^+\right] + c^i_b\E\left[(Z_i-a_i)^+\right] \right),
\end{align*}
with $x^+ \triangleq \text{max}(x,0)$. While the direct cost now depends on the pooled data through state variable $k$, we note that it remains decomposable in $N$ direct costs $C_i(a,x,k) \triangleq c_v^i(a_i-x_i)+c^i_h\E[(a_i - Z_i)^+] + c^i_b\E[(Z_i-a_i)^+]$, each associated with a local warehouse $i\in\mathcal{N}$. Let $V^N_t(\bm{x},k)$ denote the optimal expected total cost over decision epochs $t,t+1,\ldots,T$, starting from state $(x,k) \in \mathcal{S}$. By the principle of optimality, $V^N_t(\bm{x},k)$ satisfies the following recursive Bellman optimality equations
\begin{align*}
V^N_t(\bm{x},k) = \min_{\bm{a} \in \mathcal{A}(\bm{x})} \Bigg\{C(\bm{a},\bm{x},k) + \E_{\bm{Z}}\bigg[  V^N_{t+1}\Big(\bm{a}-\bm{Z},k+\sum_{i\in \mathcal{N}}Z_i \Big) \bigg]  \Bigg\},
\end{align*}
with $V^N_{T}(\cdot,\cdot)\triangleq 0$.

We now formulate the corresponding alternative MDP in which the original MDP can be decomposed. For each $i\in \mathcal{N}$, we let $\tilde{V}^{N,i}_t(x,k)$ denote the optimal expected total cost over decision epochs $t, t+1, \ldots, T$, starting from state $(x,k)\in\Z \times \N_{0}$. Then, $\tilde{V}^{N,i}_t(x,k)$ satisfies the following recursive Bellman optimality equations
\begin{align}\label{decomposedMPDspare}
\tilde{V}^{N,i}_t(x,k) = \min_{a \in \mathcal{A}(x)}\Bigg\{C_i(a,x,k) + \E_{(Z,K)}\bigg[\tilde{V}^{N,i}_{t+1}\Big(a-Z,k+Z+K\Big) \bigg]  \Bigg\}, 
\end{align}
where  $Z \sim NB\Big(\alpha_0 + k, \frac{\beta_0+N\cdot t}{\beta_0 + N\cdot t + 1}\Big)$, $K \sim NB\Big( (N-1)\cdot(\alpha_0 + k), \frac{\beta_0+N\cdot t}{\beta_0 + N\cdot t + 1}\Big)$, and $\tilde{V}^{N,i}_{T}(\cdot,\cdot) \triangleq 0.$ 
Observe that the alternative formulation in \eqref{decomposedMPDspare} resembles a single spare parts warehouse problem in isolation, but where the dynamics of the system depend on the learned information of all warehouses through $k$.

We are now in the position to present our decomposition result applied to the spare parts inventory system setting. Its proof is almost verbatim the proof of Theorem \ref{decomposition} and therefore omitted.
\begin{theorem}\label{decompositionSpare}
For each $t \in \{0,1,\ldots, T\}$, we have: 
$$V^N_t(\bm{x},k) = \sum_{i\in\mathcal{N}} \tilde{V}^{N,i}_t(x_i,k), \quad \mbox{for all }(\bm{x},k)\in\mathcal{S}.$$
\end{theorem}
As before, the above decomposition result motivates us to establish structural properties of the alternative 2-dimensional MDP, which then carry over to the original, high-dimensional MDP. To that end, we first establish convexity of the value function in the inventory level before order placement. 
\begin{proposition}\label{valueConvex}
For each $t \in \{0,1,\ldots, T\}$, $k\in\N_0$, and $i\in\mathcal{N}$, the value function $\tilde{V}^{N,i}_t(x,k)$ is convex in $x$.
\end{proposition}
\begin{proof}
See \ref{proofConvex} 
\end{proof}

The above result also implies that the optimal policy of the decomposed MDP is characterized by an order-up-to structure, in which we place orders such that the inventory level after ordering reaches a certain target level (if needed). Our next result formalizes the optimality of order-up-to levels, together with their non-decreasing monotonic behavior in the state variable $k$. 
\begin{proposition}\label{optPolicy}
For each $t \in \{0,1,\ldots, T-1\}$, $k\in\N_0$, and $i\in\mathcal{N}$, there exists a single target level $\delta_i^{(k,t)} \in \Z$ such that the optimal action is $a_i=\delta_i^{(k,t)}$ if $x<\delta_i^{(k,t)}$ and $a_i=x$ otherwise. The optimal target level $\delta_i^{(k,t)}$ is non-decreasing in $k$.
\end{proposition}
\begin{proof}
See  \ref{optPolicyProof}.
\end{proof}

Proposition \ref{optPolicy} shows that the optimal target levels depend in real-time on the shared learning process across all local warehouses via the state variable $k$ in a monotonic way. This is intuitive: As we learn from the pooled data that $\lambda$ is higher (through a higher $k$) and everything else fixed, the next demand will likely take on higher values. Therefore, we should increase the target level to which we raise our spare part inventories. This monotonicity result stands in contrast to the limiting result of the optimal control limits in the condition-based maintenance setting, as described in Proposition \ref{controllimitAsymptote}. The key difference is that, unlike in the condition-based maintenance setting, there is a cost incentive that is proportional to demand realizations. This cost incentive remains proportional, even if the demand becomes very large because of a very large $k$, so that ordering quantities are monotonically non-decreasing in $k$. We conclude by noting that similar monotonicity results for Bayesian inventory systems exist in the literature, but only for single inventory systems in isolation and without any data pooling considerations \citep[e.g.][]{iglehart1964dynamic}. 

\section{Conclusion}\label{Pooling:conclusion}
We have studied the benefits of data pooling when an unknown deterioration rate that multiple systems have in common needs to be learned over a finite lifespan. We formulated this problem as a finite horizon Bayesian MDP in which learning is pooled. This formulation suffers from the well-known curse of dimensionality. As a remedy, we proved a new decomposition result that establishes the equivalence between the original, highly dimensional MDP and multiple low-dimensional MDPs, enabling both structural analyses and numerical computations.  

Results of a comprehensive numerical study indicated that significant savings can be attained by pooling data to learn the a-priori unknown parameter. We also illustrated how our decomposition result can be applied to other settings, notably a set of canonical spare parts inventory systems, where an unknown but common demand rate needs to be learned across this set of inventory systems.   

For further research, it would be interesting to investigate the applicability of our decomposition result in other areas. One methodological area that seems particularly promising is a set of queues subjected to dynamic control; each with Poisson arrivals with a common but unknown rate and each queue with their own cost structure and action space. As many operations management problems (production-inventory systems, service systems, ride-sharing systems, to name a few) are modeled using (a set of) queues, we expect that, when extending such problems to settings where there is an unknown but common parameter (e.g., an arrival rate) that needs to be learned by pooling data (e.g., arrival data) while dynamically controlling the individual systems, our decomposition result might be useful for tractable analyses.    

\section*{Acknowledgements} 
The authors are grateful to the editorial team whose comments greatly improved the paper.

\bibliographystyle{abbrvnat}
\bibliography{sample}

\appendix
\section{Proofs}
\subsection{Proof of Lemma \ref{increasingStochastic}}\label{proofLemma1}
\begin{proof}
Suppose $k^+\geq k^->0$, keep $N$ and $t$ fixed, and let $Z(k^+)$ and $Z(k^-)$ be the posterior predictive random variable associated with $k^+$ and $k^-$, respectively. We need to show that for any increasing function $\phi(\cdot)$,  $\mathbb{E}\left[ \phi\left( Z(k^+)\right)\right] \geq \mathbb{E}\left[ \phi\left( Z(k^-)\right)\right]$ \citep[cf.][Definition 1.A.1]{shaked2007stochastic}. 
To this end, let $\bar{k}\triangleq k^+ - k^-\geq 0$. Since   $Z(k^+)$ and $Z(k^-)$ are Negative Binomial random variables with the same success probability (because $N$ and $t$ are kept fixed), we can construct another Negative Binomial random variable, denoted with  $Z(\bar{k})$, independent of $Z(k^-)$,  such that $Z(k^+) \sim Z(k^-) + Z(\bar{k})$. We have $\mathbb{E}\left[ \phi\left( Z(k^+) \right)\right] = \mathbb{E}\left[ \phi\left( Z(k^-) +Z(\bar{k}) \right)\right] \geq \mathbb{E}\left[ \phi\left( Z(k^-) \right)\right]$, where the inequality is because $\phi(\cdot)$ is increasing and $Z(\bar{k})$ is a non-negative random variable. 
\end{proof}

\subsection{Proof of Proposition \ref{valuemonotoneX}}\label{proofOfPropValueCBM}
\begin{proof}
We first prove part $(i)$ using induction on $t$. Let $i\in \mathcal{N}$, and note that the terminal costs are non-decreasing in $x$ since $\tilde{V}^{N,i}_T(x,k) = 0$ when $x<\xi_i$ and $c_u^i>0$ when $x\geq\xi_i$, so that the statement holds for the base case $T$.  Assume that the statement holds for some $0<t+1\leq T$, we will show that the statement then also holds for $t$. 
Consider $\tilde{V}^{N,i}_t(x,k)$. Since all terms, except the action of leaving the component in operation, on the right-hand side of \eqref{eq:decomposedAlt} are constant with respect to $x$, we only need to consider $\E_{(Z,K)}\bigg[\tilde{V}^{N,i}_{t+1}\Big(x+Z,k+Z+K\Big) \bigg]$. 
Let $\geq_{\text{st}}$ denote the usual stochastic order. Since the random variable $Z$ does not depend on $x$, we have $x^+ + Z \geq_{\text{st}} x^- + Z$ for $x^+\geq x^-\geq0$. Hence, together with the induction hypothesis, we conclude that the  expectation $\E_{(Z,K)}\bigg[\tilde{V}^{N,i}_{t+1}\Big(x+Z,k+Z+K\Big) \bigg]$ is non-decreasing in $x$ \citep[cf.][Theorem 1.A.3. (a)]{shaked2007stochastic}. Since the other terms are constant with respect to $x$, we conclude that  $\tilde{V}^{N,i}_t(x,k)$ is non-decreasing in $x$.

We now proceed with the proof of part $(ii)$, again using induction on $t$. Let $i\in \mathcal{N}$, and note that the terminal costs are non-decreasing in $k$ since $\tilde{V}^{N,i}_T(x,k)$ is constant with respect to $k$ for all $x$, so that the statement holds for the base case $T$. Assume that the statement holds for some $0<t+1\leq T$, we will show that the statement then also holds for $t$, i.e. for $\tilde{V}^{N,i}_t(x,k)$. Since $\tilde{V}^{N,i}_{t+1}(x,k)$ is non-decreasing in $k$ by the induction hypothesis, and non-decreasing in $x$ by Assertion $(i)$, and because the random variables $Z$ and $K$ are stochastically increasing in $k$ by Lemma \ref{increasingStochastic}, we can conclude that the expectations $ \E_{(Z,K)}\bigg[\tilde{V}^{N,i}_{t+1}\Big(Z,k+Z+K\Big) \bigg]$ and  $\E_{(Z,K)}\bigg[\tilde{V}^{N,i}_{t+1}\Big(x+Z,k+Z+K\Big) \bigg]$ are non-decreasing in $k$ as well \citep[cf.][Theorem 1.A.3. (b)]{shaked2007stochastic}. Since all terms on the right of \eqref{eq:decomposedAlt}  are non-decreasing in $k$, we conclude that $\tilde{V}^{N,i}_t(x,k)$ is non-decreasing in $k$.
\end{proof}

\subsection{Proof of Proposition \ref{controllimit}}\label{proofCL}
\begin{proof}
Preventive maintenance at decision epoch $t\in\{1,2,\ldots,T-1\}$ is optimal for $i\in \mathcal{N}$ when the following equation holds:
\begin{align}\label{inequalityOpt}
c^i_p +\E_{(Z,K)}\bigg[\tilde{V}^{N,i}_{t+1}\Big(Z,k+Z+K\Big) \bigg] \leq  \E_{(Z,K)}\bigg[\tilde{V}^{N,i}_{t+1}\Big(x+Z,k+Z+K\Big) \bigg] 
\end{align}
The left-hand side of Inequality \eqref{inequalityOpt} is constant with respect to $x$. Since $\E_{(Z,K)}\bigg[\tilde{V}^{N,i}_{t+1}\Big(x+Z,k+Z+K\Big) \bigg] $ is non-decreasing in $x$ (cf. Proposition \ref{valuemonotoneX}), we find that the right-hand side of Inequality \eqref{inequalityOpt} is non-decreasing in $x$. Hence, if the optimal decision is to carry out preventive maintenance in state $(\delta_i^{(k,t)},k)$ at decision epoch $t\in\{1,2,\ldots,T-1\}$, then the same decision is optimal for any state $(x,k)$ at decision epoch $t$ with $x \geq  \delta_i^{(k,t)}$, which implies the control limit policy.
\end{proof}

\subsection{Proof of Proposition \ref{controllimitAsymptote}}\label{proofk}
\begin{proof}
We need to show that preventive maintenance is never optimal when $k\rightarrow \infty$, for all $t\in\{1,2,\ldots,T-1\}$  and $i\in \mathcal{N}$. Observe that preventive maintenance at decision epoch $t\in\{1,2,\ldots,T-1\}$ is not optimal for $i\in \mathcal{N}$ when the following equation holds for all $x<\xi_i$:
\begin{align*}
c^i_p > \E_{(Z,K)}\bigg[\tilde{V}^{N,i}_{t+1}\Big(x+Z,k+Z+K\Big) \bigg] - \E_{(Z,K)}\bigg[\tilde{V}^{N,i}_{t+1}\Big(Z,k+Z+K\Big) \bigg] = \E_{(Z,K)}\bigg[\tilde{V}^{N,i}_{t+1}\Big(x+Z,k+Z+K\Big) -\tilde{V}^{N,i}_{t+1}\Big(Z,k+Z+K\Big) \bigg],
\end{align*}
where the equality is because both expectations are taken with respect to $(Z,K)$. To make the dependence of the random variables $Z$ and $K$ on $k$ explicit, we write $Z_k$ and $K_k$, respectively. This notation is different from other proofs in this paper -- where we use the notation $Z(k)$ and $K(k)$ -- because here we explicitly study the sequences $\{Z_k\}_{k\in\N_0}$ and $\{K_k\}_{k\in \N}$ when $k\rightarrow \infty$, in which case the use of subscript is conventional. We will show that for all $x<\xi_i$ 
\begin{align}
\lim\limits_{k\rightarrow \infty} \left( \E\bigg[\tilde{V}^{N,i}_{t+1}\Big(x+Z_k,k+Z_k+K_k\Big) - \tilde{V}^{N,i}_{t+1}\Big(Z_k,k+Z_k+K_k\Big) \bigg]\right) = 0, \label{toproof}
\end{align}
which implies the result because $c^i_p > 0$.  

Recall that $Z_k$ and $K_k$ are sums of independent (and identical) Negative Binomial random variables with the same success probability (because $N$ and $t$ are kept fixed). Let $Y\sim  NB\Big(\alpha_0, \frac{\beta_0+N\cdot t}{\beta_0 + N\cdot t + 1}\Big)$, we can then write $Z_k$ and $K_k$ as:
\begin{align*}
    Z_k \sim Y + \sum_{i=1}^k X_i, \mbox{ and, } K_k \sim (N-1) \cdot Y + (N-1)\cdot \sum_{i=1}^k X_i,  
\end{align*}
where $X_i \sim  NB\Big(1, \frac{\beta_0+N\cdot t}{\beta_0 + N\cdot t + 1}\Big)$ ($X_i$ is a Geometric random variable with success probability $p\coloneqq \frac{\beta_0+N\cdot t}{\beta_0 + N\cdot t + 1}$).  

Since $\E [ X_i^2] < \infty$ and $\E [ X_i]=\frac{1-p}{p} > 0$, it is well-known that $\lim\limits_{k\rightarrow \infty}  \sum_{i=1}^k X_i \overset{a.s.}{\to} \infty$, where $\overset{a.s.}{\to}$ denotes almost sure convergence. This follows from the Strong Law of Large Numbers: $\lim\limits_{k\rightarrow \infty} \frac{1}{k}\sum_{i=1}^k X_i \overset{a.s.}{\to} \E [ X_i]$ implies $\lim\limits_{k\rightarrow \infty}  \sum_{i=1}^k X_i \overset{a.s.}{\to}  \infty$ because $\E [ X_i]>0$. From this it follows that both  $Z_k$ and $K_k$ $\overset{a.s.}{\to} \infty$ as $k\rightarrow \infty$, and hence that:
\begin{align}
x + Z_k  &\overset{a.s.}{\to} \infty\mbox{, for all } x<\xi_i, \mbox{and,}\label{asx}  \\ 
k + Z_k + K_k &\overset{a.s.}{\to} \infty, \label{as} 
\end{align}
as $k\rightarrow \infty$. Because $\tilde{V}^{N,i}_{t+1}:\N_0^2 \rightarrow [0,\infty)$ is a (bounded) continuous function on its domain $\N_0^2$, we can invoke the Continuous Mapping Theorem \cite[see, e.g.,][Theorem 25.7]{billingsley} together with the almost sure convergence in \eqref{asx} and \eqref{as} to conclude that:
\begin{align*}
\tilde{V}^{N,i}_{t+1}\Big(x+Z_k,k+Z_k+K_k\Big)  &\overset{a.s.}{\to} \tilde{V}^{N,i}_{t+1}\Big(\infty,\infty \Big)\mbox{, for all } x<\xi_i, \mbox{and,}\\ 
\tilde{V}^{N,i}_{t+1}\Big(Z_k,k+Z_k+K_k\Big)  &\overset{a.s.}{\to} \tilde{V}^{N,i}_{t+1}\Big(\infty,\infty \Big),
\end{align*}
as $k\rightarrow \infty$, and hence that (because $\tilde{V}^{N,i}_{t+1}:\N_0^2 \rightarrow [0,\infty)$ is bounded):
\begin{align}
\tilde{V}^{N,i}_{t+1}\Big(x+Z_k,k+Z_k+K_k\Big) - \tilde{V}^{N,i}_{t+1}\Big(Z_k,k+Z_k+K_k\Big)   &\overset{a.s.}{\to} 0 \mbox{, for all } x<\xi_i, \mbox{as }k \rightarrow \infty. \label{convergence0}
\end{align}

As $\tilde{V}^{N,i}_{t+1}\Big(\cdot,\cdot \Big)$ is uniformly bounded by $(T-t)\cdot c^i_u$ from above and by $c^i_p$ from below, we know that  $\big| \tilde{V}^{N,i}_{t+1}\Big(x+Z_k,k+Z_k+K_k\Big) - \tilde{V}^{N,i}_{t+1}\Big(Z_k,k+Z_k+K_k\Big) \big| \leq (T-t)\cdot c^i_u $ for all $k$. Because $\tilde{V}^{N,i}_{t+1}\Big(x+Z_k,k+Z_k+K_k\Big) - \tilde{V}^{N,i}_{t+1}\Big(Z_k,k+Z_k+K_k\Big)$ is thus dominated for all $k$, we can apply the Dominated Convergence Theorem \citep[see, e.g.,][Theorem 16.4]{billingsley}  to exchange the order of taking the limit and expectation in \eqref{toproof} and use the convergence in \eqref{convergence0} to conclude that, for all $x<\xi$, $t\in\{1,2,\ldots,T-1\}$, and $i\in \mathcal{N}$: 
\begin{align*}
\lim\limits_{k\rightarrow \infty} \left( \E\bigg[\tilde{V}^{N,i}_{t+1}\Big(x+Z_k,k+Z_k+K_k\Big) - \tilde{V}^{N,i}_{t+1}\Big(Z_k,k+Z_k+K_k\Big) \bigg]\right) = \qquad \qquad \qquad \qquad  \\ 
 \qquad \qquad \qquad \qquad   \E\bigg[ \lim\limits_{k\rightarrow \infty} \left( \tilde{V}^{N,i}_{t+1}\Big(x+Z_k,k+Z_k+K_k\Big) - \tilde{V}^{N,i}_{t+1}\Big(Z_k,k+Z_k+K_k\Big) \right) \bigg] = 0.
\end{align*}
\end{proof}

\subsection{Proof of Proposition \ref{valueConvex}\label{proofConvex}}
Before we present the proof, we introduce some additional notation. 
For each $i\in \mathcal{N}$ and $t\in\{0,1,\ldots,T-1\}$, we find it convenient to define
\begin{align*}
    G^{N,i}_t(a,k) \triangleq c_v^ia + c^i_h\E_{Z}[(a - Z)^+] + c^i_b\E_Z[(Z-a)^+]  + \E_{(Z,K)}\bigg[\tilde{V}^{N,i}_{t+1}\Big(a-Z,k+Z+K\Big) \bigg],
\end{align*}
so that we can rewrite the recursive Bellman optimality equations in (\ref{decomposedMPD}) as
\begin{align}\label{decomposedMDP_appendix}
\tilde{V}^{N,i}_t(x,k) = \min_{a \in \mathcal{A}(x)}\Bigg\{G_t^{N,i}(a,k) - c_v^ix\Bigg\}.
\end{align}

Next, we define for each $k\in\Z$ and $i\in\mathcal{N}$ the discrete derivatives $\Delta G^{N,i}_t(a, k) \triangleq G^{N,i}_t(a+1, k) - G^{N,i}_t(a, k)$ and $\Delta^2 G^{N,i}_t(a,k)\triangleq \Delta G^{N,i}_t(a+1, k) - \Delta G^{N,i}_t(a, k)$ for each $t\in\{0,1,\ldots,T-1\}$. Similarly, for each $k\in\Z$ and $i\in\mathcal{N}$, we define $\Delta\tilde{V}^{N,i}_{t}(x,k) \triangleq \tilde{V}^{N,i}_{t}(x+1,k) - \tilde{V}^{N,i}_{t}(x,k)$ and $\Delta^2\tilde{V}^{N,i}_{t}(x,k) \triangleq \Delta\tilde{V}^{N,i}_{t}(x+1,k) - \Delta\tilde{V}^{N,i}_{t}(x,k)$ for each $t\in\{0,1,\ldots,T\}$.
\begin{proof}
We prove the statement using induction on $t$. 
Let $i\in\mathcal{N}$. Because $\tilde{V}^{N,i}_{T}(\cdot,\cdot) = 0$, the statement trivially holds for the base case $T$. 
Assume that the statement holds for some $0<t+1\leq T$, we will show that the statement then also holds for $t$. Consider $\tilde{V}^{N,i}_t(x,k)$. We find
\begin{align}
    \Delta G^{N,i}_t(a,k) &= c^i_v(a + 1) - c^i_v a + c^i_h \Bigg(\sum_{i=0}^{a+1}(a+1-i)P(Z=i) - \sum_{i=0}^{a}(a-i)P(Z=i) \Bigg) \nonumber \\
                   & + c^i_b \Bigg(\sum_{i={a+1}}^\infty(i-a-1)P(Z=i) - \sum_{i=a}^{\infty}(i-a)P(Z=i) \Bigg)  + \E_{(Z,K)}\bigg[\Delta\tilde{V}^{N,i}_{t+1}\Big(a-Z,k+Z+K\Big) \bigg], 
                    \label{eq:convexSpare}  \\ 
                    &= c^i_v + c^i_h - ( c^i_h + c^i_b)\cdot P(Z>a) + \E_{(Z,K)}\bigg[\Delta\tilde{V}^{N,i}_{t+1}\Big(a-Z,k+Z+K\Big) \bigg], \nonumber 
\end{align} 
and 
\begin{align*}
    \Delta^2 G^{N,i}_t(a,k) & = \Delta G^{N,i}_t(a+1,k) - \Delta G^{N,i}_t(a,k),  \\
    & = ( c^i_h + c^i_b)\cdot P(Z=a+1) + \E_{(Z,K)}\bigg[\Delta^2\tilde{V}^{N,i}_{t+1}\Big(a-Z,k+Z+K\Big) \bigg].
\end{align*} 
By the induction hypothesis, $\tilde{V}^{N,i}_{t+1}(a-z,k+z+k)$ is convex in $a$ for any $z$ and $k$, and because convexity is preserved under expectation,  $\E_{(Z,K)}[\tilde{V}^{N,i}_{t+1}(a-Z,k+Z+K)]$ is convex in $a$ as well. Thus $\Delta^2 G^{N,i}_t(a,k)\geq 0$ and hence $G^{N,i}_t(a,k)$ is convex in $a\in\Z$. This also implies that the function $G^{N,i}_t(a,k) - c^i_v x$ is convex in $a$ and $x$ for any $k$. Since taking the minimum of a convex function over a convex set results in a convex function, we conclude that $\tilde{V}^{N,i}_{t}(x,k) = \min_{a\geq x} \{G^{N,i}_t(a,k) - c^i_v x \}$ is a convex function of $x$.
\end{proof}

\subsection{Proof of Proposition \ref{optPolicy}}\label{optPolicyProof}
\begin{proof}
We first prove the optimality of order-up-to policies. Let $t\in\{0,1,\ldots, T-1\}$  and $i\in\mathcal{N}$. Since $\tilde{V}^{N,i}_t(x,k)$ is convex in $x$, it follows from Equation (\ref{decomposedMDP_appendix}) that the smallest $a$ for which $\Delta G^{N,i}_t(a,k)\geq 0$ is the unconstrained minimizer of $\tilde{V}^{N,i}_t(x,k)$, denoted $\delta_i^{(k,t)}$. This minimizer is independent of $x$, except that $x$ serves as lower bound in $\tilde{V}^{N,i}_t(x,k)$. Hence, we have
\begin{align*} 
\tilde{V}^{N,i}_{t}(x,k) = \begin{cases}
  G_{t+1}^{N,i}(\delta_i^{(k,t)},k) - c_v^i x&  \text{if }x \leq   \delta_i^{(k,t)}, \\
   G_{t}^{N,i}(x,k) - c_v^i x& \text{if } x >  \delta_i^{(k,t)},
\end{cases}
\end{align*}
which shows that $G_{t+1}^{N,i}(\delta_i^{(k,t)},k)$ remains constant as $x\leq \delta_i^{(k,t)}$, and increases for $x >  \delta_i^{(k,t)}$. Hence, the optimal action in state $(x,k)\in\mathcal{S}$ at decision epoch $t$ is $a_i=\delta_i^{(k,t)}$ if $x<\delta_i^{(k,t)}$ and $a_i=x$ otherwise.

We next continue with the monotonic behavior of the optimal target levels in $k$.
Our proof shares similarities with that of \cite{iglehart1964dynamic}, but his proof is for continuous demand distributions and without any data pooling considerations. We prove the statement using induction on $t$, and do so together with two auxiliary results:

\begin{proposition}\label{aux}
\begin{enumerate}[label=(\roman*)] For each $t\in\{0,1,\ldots,T-1\}$ and $i\in\mathcal{N}$, 
\item $\Delta G^{N,i}_t(a, k)$ is non-increasing in $k$, and
\item $\Delta \tilde{V}^{N,i}_t(x, k)$ is non-increasing in $k$.
\end{enumerate}
\end{proposition}
Let $k^+,k^-\in\N_0$ with $k^+>k^-$, and $i\in\mathcal{N}$. Similar to the first part of our derivation in Equation (\ref{eq:convexSpare}), we can write, using $\tilde{V}^{N,i}_{T}(\cdot,\cdot) = 0$:
\begin{align*}
    \Delta G^{N,i}_{T-1}(a,k) & = c^i_v + c^i_h - ( c^i_h + c^i_b)\cdot P(Z>a).
\end{align*}
Recalling from Lemma \ref{increasingStochastic} that the random variable $Z$ is stochastically increasing in $k$, we find that $\Delta G^{N,i}_{T-1}(a,k^+) \leq \Delta G^{N,i}_{T-1}(a,k^-)$ with $k^+\geq k^-$. This proves the base case for part $(i)$ of Proposition \ref{aux}. Following a similar reasoning and noting that $\delta_i^{(k,T-1)}$ is the smallest $a$ for which $\Delta G^{N,i}_{T-1}(a,k)\geq 0$ due to convexity of $\tilde{V}^{N,i}_{T-1}(x,k)$, we also find that $\delta_i^{(k^+,T-1)}\geq \delta_i^{(k^-,T-1)}$ with $k^+\geq k^-$. This proves the base case for the monotonic behavior of the order-up-to levels in Proposition \ref{optPolicy}. Since the optimal policy has an order-up-to structure, we can write 
\begin{align} \label{deltaValueBase}
\Delta\tilde{V}^{N,i}_{T-1}(x,k) = \begin{cases}
  - c_v^i  &  \text{if }x < \delta_i^{(k,T-1)}, \\
  \Delta G^{N,i}_{T-1}(x,k) - c_v^i & \text{if } x \geq  \delta_i^{(k,T-1)},
\end{cases}
\end{align}
which yields the base case $\Delta\tilde{V}^{N,i}_{T-1}(x,k^+)  \leq \Delta \tilde{V}^{N,i}_{T-1}(x,k^-)$ for part $(ii)$ of Proposition \ref{aux} since $\delta^{(k^+,T-1)} \geq \delta^{(k^-,T-1)}$.

Assume that $(i)-(ii)$ of Proposition \ref{aux} as well as $\delta_i^{(k^+,t)}\geq \delta_i^{(k^-,t)}$ of Proposition \ref{optPolicy} hold for $0<t+1\leq T-1$, we will now show that they then also hold for $t$. We continue from Equation (\ref{eq:convexSpare}) and write $Z(k)$ and $K(k)$ to denote the dependence of $Z$ and $K$ on $k$ explicitly:
\begin{align*}
\Delta G^{N,i}_t(a,k^-) & = c^i_v + c^i_h - ( c^i_h + c^i_b)\cdot P(Z(k^-)>a)  + \E_{(Z(k^-),Z(k^-))}\bigg[\Delta\tilde{V}^{N,i}_{t+1}\Big(a-Z(k^-),k^-+Z(k^-)+K(k^-)\Big) \bigg], \\ 
   &  \geq c^i_v + c^i_h - ( c^i_h + c^i_b)\cdot P(Z(k^+)>a) + \E_{(Z(k^-),Z(k^-))}\bigg[\Delta\tilde{V}^{N,i}_{t+1}\Big(a-Z(k^-),k^++Z(k^-)+K(k^-)\Big) \bigg], \\
   & \geq c^i_v + c^i_h - ( c^i_h + c^i_b)\cdot P(Z(K^+)>a) + \E_{(Z(k^+),Z(k^+))}\bigg[\Delta\tilde{V}^{N,i}_{t+1}\Big(a-Z(k^+),k^++Z(k^+)+K(k^+)\Big) \bigg], \\
   & = \Delta G^{N,i}_t(a,k^+),
\end{align*}
The first inequality follows from the induction hypothesis and the fact that $Z$ is stochastically increasing in $k$ by Lemma \ref{increasingStochastic}.
Because of the convexity of $\tilde{V}^{N,i}_{t+1}(x,k)$ in $x$ and the induction hypothesis we conclude that $\Delta\tilde{V}^{N,i}_{t+1}(a-z,k)$ is a non-increasing function of $z$ and $k$ for any $a$. Combining this with the fact that $Z$ and $K$ are stochastically increasing in $k$ by Lemma \ref{increasingStochastic} gives the second inequality \citep[cf.][Theorem 1.A.3. (b)]{shaked2007stochastic}. Similar to our reasoning in the base case, this also implies that $\delta^{(k^+,t)} \geq \delta^{(k^-,t)}$, and together with Equation (\ref{deltaValueBase}) for $t$, yields $\Delta\tilde{V}^{N,i}_{t}(x,k^+)  \leq \Delta \tilde{V}^{N,i}_{t}(x,k^-)$.
\end{proof}
\end{document}